\documentclass{article}

\usepackage[nonatbib]{mystyle}

\usepackage{microtype}
\usepackage{graphicx}
\usepackage{subcaption}
\usepackage{booktabs} %

\usepackage{hyperref}

\usepackage[preprint]{icml2026}

\usepackage{amsmath}
\usepackage{amssymb}
\usepackage{mathtools}
\usepackage{amsthm}

\usepackage[capitalize,noabbrev]{cleveref}

\usepackage[textsize=tiny]{todonotes}

\icmltitlerunning{Prior Diffusiveness and Regret in the Linear-Gaussian Bandit}

\begin{document}

\twocolumn[
  \icmltitle{Prior Diffusiveness and Regret in the  Linear-Gaussian Bandit}

  \icmlsetsymbol{equal}{*}

  \begin{icmlauthorlist}
    \icmlauthor{Yifan Zhu}{stanford}
    \icmlauthor{John C. Duchi}{stanford}
    \icmlauthor{Benjamin Van Roy}{stanford}
  \end{icmlauthorlist}

  \icmlaffiliation{stanford}{Department of Electrical Engineering, Stanford University, California, United States}

  \icmlcorrespondingauthor{Yifan Zhu}{zhuyifan@stanford.edu}

  \icmlkeywords{Machine Learning, ICML}

  \vskip 0.3in
]

\printAffiliationsAndNotice{}  %

\begin{abstract}
  We prove that Thompson sampling exhibits \(\Otilde(\sigma d \sqrt{T} + d r
  \sqrt{\tr(\Sigma_0)})\) Bayesian regret in the linear-Gaussian bandit with
  a \(\normal(\mu_0, \Sigma_0)\) prior distribution on the coefficients,
  where \(d\) is the dimension, \(T\) is the time horizon, \(r\) is the
  maximum \(\ell_2\) norm of the actions, and \(\sigma^2\) is the noise
  variance.
  In contrast to existing regret bounds, this shows that
  to within logarithmic factors, the
  prior-dependent ``burn-in'' term
  \(d r \sqrt{\tr(\Sigma_0)}\) decouples additively from the minimax (long run)
  regret
  \(\sigma d \sqrt{T}\).
  Previous regret bounds exhibit a
  multiplicative dependence on these terms.
  We establish these results via a new ``elliptical potential'' lemma, and
  also provide a lower bound indicating that the burn-in term is
  unavoidable.
\end{abstract}

\section{Introduction}

In the linear-Gaussian bandit, at times $t = 0, 1, 2, \ldots$, a player
sequentially chooses an action $A_t \in \Ac \subset \R^d$ and
receives a reward
\begin{equation}
  \label{eqn:linear-gaussian-rewards}
  R_{t + 1} = \thetaopt^\top A_t + W_{t + 1},
  ~~~
  W_{t + 1} \simiid \normal(0, \sigma^2),
\end{equation}
where $\Ac$ denotes the action set and $\thetaopt$ an unknown vector, and
the goal is to maximize the rewards.
Thompson, or posterior, sampling~\citep{Thompson33} has proven surprisingly
(empirically) effective for this problem.
It proceeds by first placing a prior distribution $\prior$ on
$\thetaopt$.
At time $t$, conditional on the history
\begin{equation*}
  \history_t \defeq (A_0, R_1, \ldots, A_{t-1}, R_t)
\end{equation*}
of actions and rewards to time $t$, Thompson sampling draws
\begin{equation*}
  \hat{\theta}_t \sim \prior(\cdot \mid \history_t),
\end{equation*}
that is, a vector from the posterior on $\thetaopt$ given $\history_t$,
and then takes the (putatively) optimal action conditional
on $\hat{\theta}_t$ via
\begin{equation}
  \label{eqn:optimal-action}
  a\subopt(\theta) \in \argmax_{a \in \Ac} \theta^\top a,
  ~~~ \mbox{i.e.} ~~~
  A_t = a\subopt(\hat{\theta}_t).
\end{equation}
An optimal action is $A\subopt \defeq a\subopt(\thetaopt)$, and in this
framework, we study the Bayesian regret
\begin{equation}
  \label{eqn:regret}
  \regret(T) := \sum_{t=0}^{T-1} \Eb \left[
  \thetaopt^{\top} A\subopt - R_{t+1}
  \right],
\end{equation}
where the expectation is taken over the random $\thetaopt \sim \prior$, the random noise $W_t$, and
any randomness in the actions chosen.

When the prior $\prior$ is Gaussian, so that $\thetaopt \sim \normal(0,
\Sigma_0)$, letting $\bA_t = [A_0 ~ \cdots ~ A_{t-1}]^\top \in \R^{t \times
  d}$ be the matrix of actions and $\bR_t = [R_1 ~ \cdots ~ R_t]^\top$ the
vector of rewards, the posterior for the linear
model~\eqref{eqn:linear-gaussian-rewards} takes a particularly simple form.
Indeed, for
\begin{equation*}
  \Sigma_t = \left(\sigma^{-2} \bA_t^\top \bA_t + \Sigma_0^{-1}\right)^{-1},
\end{equation*}
we have
\begin{equation*}
  \thetaopt \mid \history_t
  \sim
  \normal\left(
  \sigma^{-2} \Sigma_t \bA_t^\top \bR_t,
  \Sigma_t \right),
\end{equation*}
so that sampling $\hat{\theta}_t$ is easy, and typically (e.g., if $\Ac$ is
a scaled $\ell_2$-ball), so is sampling the actions~\eqref{eqn:optimal-action}.

The ``canonical'' linear-Gaussian bandit takes the action set
$\Ac \subset r \ball_2^d$, the $\ell_2$-ball of radius $r$.
For this canonical setting, \citet{KalkanliOz20} provide what we consider a
prototypical regret bound, developing an information-theoretic analysis to
demonstrate that Thompson sampling enjoys the regret bound
\begin{equation*}
  \regret(T)
  \lesssim  d
  \sqrt{T (\sigma^2 + r^2 \tr(\Sigma_0) ) \log (1 + T/d)}.
\end{equation*}
In this bound, the \emph{diffusiveness} $\sqrt{\tr(\Sigma_0)}$ of the prior
multiplies the asymptotic minimax rate $d \sqrt{T}$, a looseness that many
other regret bounds suffer and that we show is unnecessary.

We separate the minimax terms from a bound on the ``burn-in'', establishing
an \(\Otilde(\sigma d \sqrt{T} + d r \sqrt{\tr(\Sigma_0)})\) regret bound
(Section~\ref{sec:analysis}), where only the observation noise $\sigma$
scales the minimax rate (and the $\Otilde$ hides terms scaling as $\log(1 +
T/d)$).
To establish this result, we prove a novel generalization of the elliptical
potential lemma (Section~\ref{sec:elliptical-potential}).
This enables a more flexible analysis of stochastic optimization.
We interpret the \(d r \sqrt{\tr(\Sigma_0)}\) term as
``burn-in'' regret that any algorithm necessarily incurrs
while reducing the
uncertainty captured in the prior distribution to the scale of the
observation noise.
We establish a lower bound in Section~\ref{sec:lower-bounds} indicating
this burn-in regret is unavoidable.

As an additional contribution, in Section~\ref{sec:generalization-strongly-log-concave}, we generalize our regret bound to address any strongly log-concave prior and noise distribution, whether or not they are Gaussian.

\paragraph{Related work}

While there is a large body of work on Thompson sampling for the linear
bandit, the majority assumes that the coefficients $\thetaopt$ have compact
support, and the dependence on the scale of $\thetaopt$ is often opaque.
Many papers assume that the model parameter is bounded, while others assume
the rewards $R_t$ are bounded, effectively implying the parameter
$\thetaopt$ is bounded in the linear
bandit~\eqref{eqn:linear-gaussian-rewards}.
We give a somewhat terse list of the most relevant papers,
describing briefly the results of their analyses, and recalling that
the asymptotic minimax lower bound for regret in the linear
bandit~\eqref{eqn:linear-gaussian-rewards} scales as
$d \sigma \sqrt{T}$~\citep{RusmevichientongTs10}.
We ignore all logarithmic factors below.

\begin{itemize}
\item \citet{AbeilleLa17} study a variant of linear Thompson sampling that
  inflates the posterior variance by $d$.
  Assuming that $r=1$ and $\ltwo{\thetaopt} \le S$ (Assumption 2), they
  prove a frequentist (non-Bayesian) regret bound of $\sigma
  d^{1.5}\sqrt{T}+Sd\sqrt{T}$.
\item \citet{AgrawalGo13} assume that $r=1$ and $\ltwo{\thetaopt} \le 1$,
  and shows a high-probability regret bound of $d^{1.5}\sqrt{T}$.
  They observe that their regret bound scales linearly with
  $\ltwo{\thetaopt}$ (see Section 2.1 of their paper), which is sub-optimal.
\item \citet{DongVa18} assume that the rewards $R_t \in [-1, 1]$, and an
  inspection of their results to allow $R_t \in [-b,b]$
  yields a Bayesian regret bound of $d b \sqrt{T}$
  (see their Proposition 3, where an application of Pinsker's inequality
  would result in their ``information ratio'' $\Gamma_t$ scaling
  as $d b^2$ rather than $d$).
  Implicitly, then, their regret bounds scale with $\ltwo{\thetaopt}$.
\item \citet{HamidiMo22} assume that $r=1$ and $\ltwo{\thetaopt} \le 1$, and
  an inspection of their proofs suggests regret bounds scale linearly in the
  potential magnitude of $b = \sup_{a \in \Ac,\thetaopt \in \Theta} a^\top
  \theta$.
  Indeed, inequality~(A.2) of the paper becomes
  \[A_t^\top \cov(\theta \mid H_t) A_t
  \le b \log(1 + A_t^\top \cov(\theta \mid H_t) A_t),
  \]
  meaning their regret (implicitly) scales with $\ltwo{\thetaopt}$.
\item
  Similar to the preceding results, \citet{RussoVa16} assume the
  rewards belong to an interval of length $1$ and the number
  of actions is finite, proving a Bayesian regret bound of
  $\sqrt{\log(|\Ac|) d T}$.
  Their analysis scales with $b = \sup a^\top \theta$,
  so while this guarantee can improve upon the typical $d
  \sqrt{T}$ rate when the action set $\Ac$ is not too large,
  it also suffers the same implicit scaling in $\ltwo{\thetaopt}$.
\item \citet{RussoVa13} provide a regret bound that is, in some sense, the
  closest to ours: they assume rewards $R_t \in [0, C]$, providing a Bayesian
  regret bound of $\sigma d \sqrt{T} + dC$ for Thompson sampling and an
  Upper Confidence Bound (UCB)-type algorithm.
  The only weakness in this result is that they do not allow (potentially)
  unbounded parameters $\thetaopt$.
\item \citet{RussoVa14c}
  allow a Gaussian prior (Proposition 5), but assume the
  action set $\Ac$ is finite, providing a Bayesian regret bound
  of $\sigma \sqrt{d T \log |\Ac|}$.
  They assume that the prior variance $\Sigma_0$ has
  diagonal bounded by $1$, and they do not explicitly discuss
  the effect of the prior on the regret.
\end{itemize}

\citet{KalkanliOz20} inspire this particular work, and, as we note in the
introduction, they give regret bound of order $\sigma d \sqrt{T} + d r
\sqrt{\tr(\Sigma_0)T }$.
We show that the ``correct'' regret bound---in that there are matching upper
and lower bounds---scales as the smaller quantity $\sigma d \sqrt{T} + d r
\sqrt{\tr(\Sigma_0)}$.
It is worth mentioning that results on Thompson sampling for Gaussian
process bandits \citep{ChowdhuryGo17,SrinivasKrKaSe10} also imply regret
bounds for the linear-Gaussian bandit.
These results face similar limitations as that of \citet{KalkanliOz20}, where
the prior diffusiveness $\sqrt{\opnorm{\Sigma_0}}$
multiplies the regret \(d \sqrt{T}\).

\section{Analysis}
\label{sec:analysis}

Our main contribution, which we present in this section, consists of a
sharper regret bound for Thompson sampling in the linear-Gaussian
bandit~\eqref{eqn:linear-gaussian-rewards}.
To remind the reader, throughout, we wish to bound the Bayesian
regret~\eqref{eqn:regret}, where we assume the canonical Gaussian bandit: we
have prior $\thetaopt \sim \normal(0, \Sigma_0)$, and assume the actions
$\Ac \subset r \ball_2^d$, the $\ell_2$-ball in $\R^d$ of radius $r$.
To state the theorem precisely, we require two constants that depend
at worst logarithmically on $T$, $r$, $\opnorm{\Sigma_0}$, $\sigma^{-1}$, and $1/d$:
\begin{align*}
  & C_1(d, T) \defeq
  \sqrt{1 + \max\Big\{\frac{24 \log T}{d},
    \sqrt{\frac{24 \log T}{d}}\Big\}} \\
  & C_2(d, T, \sigma, r, \Sigma_0) 
  \defeq 2 C_1(d, T) 
  \sqrt{\log\Big(1 + \frac{r^2 \opnorm{\Sigma_0} T}{d \sigma^2}\Big)}.
\end{align*}

\begin{theorem}
  \label{theorem:main-regret}
  Let the preceding assumptions and constant definitions hold.
  Then Thompson sampling satisfies the Bayesian regret bound
  \begin{align*}
    \regret(T)
    & \le
    d \sigma \sqrt{T} \cdot C_2(d, T, \sigma, r, \Sigma_0) \\
    & \quad
    + 3\sqrt{2} r \sqrt{d} \tr(\Sigma_0^{1/2}) C_1(d, T)
    + \sqrt{2 r^2 \tr(\Sigma_0)}.
  \end{align*}
\end{theorem}

Because \(C_1\) and \(C_2\) only depend logarithmically on \(T\),
\(\sigma^{-1}\), \(r\), and \(\opnorm{\Sigma_0}\), we establish the
following corollary:

\begin{corollary}
  \label{corollary:main-regret}
  Let the assumptions of Theorem~\ref{theorem:main-regret} hold.
  Then
  \begin{equation*}
    \regret(T)
    = \Otilde\left(\sigma d \sqrt{T} + d r \sqrt{\tr(\Sigma_0)}\right).
  \end{equation*}
\end{corollary}
\begin{proof}
  Applying the Cauchy-Schwarz inequality to the trace term in
  Theorem~\ref{theorem:main-regret} yields
  \begin{equation*}
    \tr \left( \Sigma_0^{1/2}\right)
    \le \sqrt{\tr(I_d)
      \tr \left( \Sigma_0\right)
    }
    = \sqrt{d \tr \left( \Sigma_0\right)}.
    \qedhere
  \end{equation*}
\end{proof}

Since \(\Eb[\ltwo{\thetaopt}^2]= \tr(\Sigma_0)\), the regret
incurred by a very suboptimal action is on the order of \(r
\sqrt{\tr(\Sigma_0)}\).
Thus, the \(d r \sqrt{\tr(\Sigma_0)}\) term in the regret bound represents
the initial exploration cost over all \(d\) dimensions.
After this burn-in, the concentration of the noise ensures that the
posterior localizes to a region predominantly determined by \(\sigma\),
so that the regret in later rounds is of order \(\sigma d
\sqrt{T}\), regardless of prior diffusiveness.

\subsection{Proof of Theorem~\ref{theorem:main-regret}}
\label{sec:proof-of-thm-regret-bound}

To simplify notation, define the conditional mean and variance functions
\(\Eb_t[\ \cdot\ ] := \Eb[\ \cdot\ \mid H_t]\), \(\Vb_t[\ \cdot\ ] := \var(
\cdot \mid H_t)\), and let \(V_t:=\Vb_t[\thetaopt]^{-1}\) be the inverse
posterior variance (precision matrix).
Then (linear-Gaussian) Thompson sampling yields posterior precision matrix
\begin{align}\label{eq:posterior-precision}
  V_t &= \Sigma_0^{-1} + \frac{1}{\sigma^2}\sum_{i=0}^{t-1} A_i A_i^{\top}
  .
\end{align}

With this, we may decompose the expected instantaneous regret conditioned on
the history \(H_t\)
via
\begin{align*}
  \lefteqn{\Eb_t \left[\theta\subopt^{\top}  A\subopt - R_{t+1} \right]}
  \\
  & = \Eb_t \left[\theta\subopt^{\top} \left( A\subopt - A_t \right)\right]
  \\
  & =
  \Eb_t \left[\theta\subopt^{\top}  A\subopt - \hat{\theta}_t^{\top} A_t  \right]
  + \Eb_t \left[ \left(\hat{\theta}_t - \theta\subopt \right)^{\top} A_t  \right]
  \\
  & = 
  \Eb_t \left[\theta\subopt^{\top}  a\subopt(\theta\subopt) - \hat{\theta}_t^{\top} a\subopt(\hat{\theta}_t)  \right]
  + \Eb_t \left[ \left(\hat{\theta}_t - \theta\subopt \right)^{\top} A_t  \right]
         .
\end{align*}
As \(\theta\subopt\) and \(\hat{\theta}_t\) have the same distribution
conditioned on \(H_t\), the first term is \(0\).
Thus
\begin{align}
  \label{eqn:instantaneous-regret}
  \Eb_t \left[\theta\subopt^{\top}  A^t\subopt - R_{t+1} \right]
  & =
  \Eb_t \left[ \left(\hat{\theta}_t - \theta\subopt \right)^{\top} A_t  \right]
  .
\end{align}

Now, for shorthand, let $\norm{x}_B^2 = x^\top B x$ be the Mahalanobis norm
associated to $B$, and
for \(\beta\) to be determined, define the events
\begin{align*}
  E_t(\beta) &:= \left\{ \|\hat{\theta}_t - \theta\subopt\|_{V_t} \leq \beta  \right\},
  \\
  E(\beta) &:= \bigcap_{t=0}^{T-1} E_t(\beta),
\end{align*}
which correspond to the sampled $\hat{\theta}_t$ being close (in the
appropriate posterior variance metric) to $\thetaopt$.
By equation~\eqref{eqn:instantaneous-regret}, we may decompose the regret
into
\begin{equation}
  \label{eqn:split-regret-via-events}
  \begin{split}
    \regret(T)
    & =
    \underbrace{\Eb \left[
        \bm{1}\{E(\beta)\} \sum_{t=0}^{T-1} (\hat{\theta}_t - \theta\subopt)^{\top} A_t
        \right]}_{(I)}
    \\
    & \quad + \underbrace{\Eb \left[
        \bm{1}\left\{E(\beta)^{\complement}\right\} \sum_{t=0}^{T-1} (\hat{\theta}_t - \theta\subopt)^{\top} A_t
        \right]}_{(II)} %
  \end{split}
\end{equation}
We control these terms in the remainder of the proof.

\subsubsection{Controlling term~(I) of
  the regret~\eqref{eqn:split-regret-via-events}}

Applying the Cauchy-Schwarz inequality
gives
\begin{align}
  (I) &\le
  \Eb \left[
    \bm{1}\{E(\beta)\} \sum_{t=0}^{T-1} \|\hat{\theta}_t - \theta\subopt\|_{V_t} \| A_t \| _{V_t^{-1}}
    \right]
  \nonumber \\
  &\le
  \beta\Eb \left[
    \sum_{t=0}^{T-1} \| A_t \|_{V_t^{-1}} \right].
  \label{eqn:time-to-apply-elliptical}
\end{align}
To control the sum $\sum_t \norm{A_t}_{V_t^{-1}}$, we leverage the particular
structure of the precision $V_{t + 1}$ as a sum of rank-one updates
involving $A_t A_t^\top$, using a new \emph{elliptical potential}
lemma.
Because the proof of the lemma is involved, we state it here, deferring
to Section~\ref{sec:elliptical-potential} its proof and 
commentary on its applications in analysis of bandit algorithms.

\begin{lemma}[Generalized elliptical potential lemma]
  \label{lemma:generalized-elliptical-potential}
  Let \(V_0\) be a positive definite matrix and
  \begin{align*}
    V_{t+1} = V_t + u_t u_t^{\top},
  \end{align*}
  where \(u_t \in \Rb^d\) satisfies \(\ltwo{u_t} \le 1\).
  Then for \(p \in [0, 1]\),
  \begin{align*}
    \sum_{t=0}^{T-1} \| u_t \|_{V_t^{-1}}^{2p}
    & \le
    2^p T^{1-p} \left(\log \frac{\det V_{T}}{\det V_0}\right)^p 
    \\
    & \qquad + \frac{3}{2p} \left(
    \tr(V_0^{-p}) - \tr(V_{T}^{-p})
    \right),
  \end{align*}
  where for $p = 0$ we take $\lim_{p \downarrow 0} \frac{1}{p} \tr(V_0^{-p}
  - V_T^{-p}) = \log \det(V_T / V_0)$.
\end{lemma}
  
By carefully choosing appropriate scaling in
inequality~\eqref{eqn:time-to-apply-elliptical}, we can apply
Lemma~\ref{lemma:generalized-elliptical-potential}.
We thus define
\begin{equation}
  \label{eq:u1-def}
  \begin{split}
    U_0 & = \frac{\sigma^2}{r^2} V_0 = \frac{\sigma^2}{r^2} \Sigma_0^{-1},
    ~~~
    u_t = \frac{A_t}{r},
    \\
    U_{t+1} &
    = U_t + u_t u_t^\top
    = \frac{\sigma^2}{r^2} V_{t+1}.
  \end{split}
\end{equation}
With these choices, $\norm{u_t}_{U_t^{-1}} = \sigma^{-1} \cdot
\norm{A_t}_{V_t^{-1}}$, and taking \(p=1/2\) in
Lemma~\ref{lemma:generalized-elliptical-potential}, we obtain
\begin{align}
  \label{eq:elliptical-potential-bound}  
  \lefteqn{\frac{1}{\sigma}
    \sum_{t = 0}^{T - 1} \norm{A_t}_{V_t^{-1}}
    =
    \sum_{t=0}^{T-1} \norm{u_t}_{U_t^{-1}}} \\
  & \le
  \sqrt{2T \log \left( \frac{\det(U_{T})}{\det(U_0)} \right)}
  + 3 \left(
  \tr\left( U_{0}^{-\frac{1}{2}} \right) 
  - \tr\left( U_{T}^{-\frac{1}{2}} \right)
  \right).
  \nonumber
\end{align}

We bound each of the terms in the bound~\eqref{eq:elliptical-potential-bound}.
Expanding out the log determinant,
we obtain
\begin{align}
  \lefteqn{\log \frac{\det(U_{T})}{\det(U_0)}}
  \notag \\
  & = \log \det \left( U_0^{-1} \left( U_0 + \sum_{t=0}^{T-1} \frac{1}{r^2} A_t A_t^{\top} \right) \right)
  \nonumber \\
  & =  \log \det \left( I_d + \frac{1}{\sigma^2} \Sigma_0 \sum_{t=0}^{T-1} A_t A_t^{\top} \right)
  \nonumber \\
  & \le
  d \log \left(\frac{1}{d}\tr\left( I_d + \frac{1}{\sigma^2} \Sigma_0 \sum_{t=0}^{T-1} A_t A_t^{\top} \right)\right)
  \nonumber \\
  & = 
  d \log \left( 1 + \frac{1}{d \sigma^2} \sum_{t=0}^{T-1} A_t^{\top} \Sigma_0 A_t \right)
  \nonumber \\
  & \le
  d \log \left(1 + \frac{r^2 \opnorm{\Sigma_0}}{d \sigma^2} \cdot T \right),
  \label{eqn:boopy-schmoops}
\end{align}
where the first inequality uses the arithmetic-geometric inequality and the
final bound trivally uses $A_t^\top \Sigma_0 A_t \le r^2 \opnorm{\Sigma_0}$.
Using definition~\eqref{eq:u1-def} of $U_0$ and
the trivial bound
$\tr(U_0^{-1/2}) - \tr(U_T^{-1/2}) \le
\frac{r}{\sigma} \tr(\Sigma_0^{1/2})$,
we substitute inequality~\eqref{eqn:boopy-schmoops} into
inequality~\eqref{eq:elliptical-potential-bound} to obtain
\begin{align*}
  \lefteqn{\sum_{t=0}^{T-1} \| u_t \|_{U_t^{-1}}} \\
  & \le
  \sqrt{2T}\sqrt{d \log \left(1 + \frac{r^2 \opnorm{\Sigma_0}}{d \sigma^2} \cdot T \right)}
  + \frac{3r}{\sigma} \tr\left( \Sigma_{0}^{\frac{1}{2}} \right) 
  .
\end{align*}
Multiplying through by $\sigma\beta$ gives the term $(I)$ bound
\begin{align}
  \label{eq:I-bound}
  \lefteqn{(I)} \\
  & \le
  \beta \sqrt{2T} \sigma \sqrt{
    d \log \left(1 + \frac{r^2 \opnorm{\Sigma_0} T}{d \sigma^2} \right)
  }
  + 3 \beta r \tr\left( \Sigma_{0}^{\frac{1}{2}} \right)
  .
  \nonumber
\end{align}

\subsubsection{Controlling term~(II) of the
    regret~\eqref{eqn:split-regret-via-events}}

We use Cauchy-Schwarz to bound the second ``small probability'' term in
equation~\eqref{eqn:split-regret-via-events} via
\begin{align*}
  (II)
  \le
  \sum_{t=0}^{T-1} \sqrt{\Eb \left[
      \bm{1}\left\{E(\beta)^{\complement}\right\}^2\right] \Eb\left[ \left((\hat{\theta}_t - \theta\subopt)^{\top} A_t\right)^2
      \right]}
  .
\end{align*}
As posterior
sampling makes \(\hat{\theta}_t\) and \(\theta\subopt\) i.i.d.\ given \(H_t\),
\begin{align*}
  \lefteqn{
    \Eb\left[ \left((\hat{\theta}_t - \theta\subopt)^{\top} A_t\right)^2 \right]  }
  \\
  & \le \Eb\left[ \|\hat{\theta}_t - \theta\subopt\|^2 \right] r^2 
  = 2 \tr\left(\Vb_t\left[\theta\subopt \right]\right) r^2
  \le 2 \tr\left(\Vb \left[\theta\subopt \right]\right) r^2
  .
\end{align*}
Therefore
\begin{align}
  (II)
  &\le
  r T \sqrt{2 \Pb\left(E(\beta)^{\complement}\right)\tr\left(\Sigma_0\right)}
  .
  \label{eq:II-bound}
\end{align}

\subsubsection{Finalizing the proof of Theorem~\ref{theorem:main-regret}}

To finalize the proof, we combine the bounds~\eqref{eq:I-bound}
and~\eqref{eq:II-bound}.
For the latter, we choose \(\beta\) to make $\Pb(E(\beta)^{\complement})$
small.
Conditioned on \(H_t\), \(\hat{\theta}_t\) and \(\theta\subopt\) are i.i.d.\
Gaussians with covariance \(V_t^{-1}\),
so \(\frac{1}{2}\|\hat{\theta}_t -
\theta\subopt\|_{V_t}^2\) has \(\chi_d^2\)-distribution.  By
standard concentration results \citep[][Ex.\ 2.11]{Wainwright19},
\begin{align*}
  \Pb \left[\chi^2_d - d \ge s\right] \le \max \left\{ e^{-s^2 / (8d) }, e^{-s/8} \right\}
\end{align*}
for $s \ge 0$.
Set \(s=\max \left\{ 24 \log T, \sqrt{24 d\log T} \right\}\). Then \(\beta =
\sqrt{2 d + 2 \max \left\{ 24 \log T, \sqrt{24 d\log T} \right\}}\) satisfies
\begin{align*}
  \Pb \left[ E_t(\beta)^{\complement} \right] \le \frac{1}{T^3} \quad \text{for all } t\in\{1,\dots,T\}.
\end{align*}
By a union bound, we have
\(
\Pb [ E(\beta)^{\complement} ] \le 1/T^2
\)    
,
whence
\begin{align}
  \label{eq:II-bound-final}
  (II) \le \sqrt{2 \tr\left(\Sigma_0\right)} r.
\end{align}
Plugging inequalities~\eqref{eq:I-bound} and \eqref{eq:II-bound-final} into
equation~\eqref{eqn:split-regret-via-events},
we obtain
\begin{multline*}
  \regret(T)
  \le
  d \sigma \sqrt{T} \sigma C_2(d, T, \sigma, r, \Sigma_0)
  \\+ 3 \sqrt{2} r \sqrt{d}\tr\left( \Sigma_{0}^{\frac{1}{2}} \right) C_1(d, T)
  + \sqrt{2 \tr\left(\Sigma_0\right)} r, %
\end{multline*}
where \( C_1 \) and \( C_2 \) are as in the theorem statement.

\section{The generalized elliptical potential lemma}
\label{sec:elliptical-potential}

The elliptical potential lemma is a standard tool in the analysis of
algorithms for linear bandits \citep[e.g.,][]{DaniHaKa08,
  Abbasi-YadkoriPaSz11,AbeilleLa17,HamidiMo22}, often arising in applying a
regret decomposition similar to that we provide in
inequality~\eqref{eqn:time-to-apply-elliptical}.
The prototypical form relies on controlling quadratic errors, and
we restate one version here:
\begin{lemma}[Elliptical potentials,
    Proposition~2 of \citet{AbeilleLa17}]
  \label{lem:elliptical-potential}
  Let \(\lambda\ge 1\) and \(u_0, u_1, \dots,
  u_{T-1} \in \Rb^d\)  satisfy \(\ltwo{u_t} \le 1\).
  Define \(V_t = \lambda I + \sum_{s=0}^{t-1} u_s u_s^{\top}\).
  Then
  \begin{align}
    \sum_{t=0}^{T-1} \| u_t \|_{V_t^{-1}}^2 \le 2 \log \frac{\det V_{T}}{\det V_0}.
  \end{align}
  
\end{lemma}

While there exist several many generalizations of the elliptical potential
lemma, we know of none directly applicable to our setting.
In particular, our generalization
(Lemma~\ref{lemma:generalized-elliptical-potential}) captures the dependence
on the initial potential $V_0$, allowing it to be nearly 0, and allows for
a more flexible exponent \(p\).
Before proving Lemma~\ref{lemma:generalized-elliptical-potential}, we note a
few related works, each of which relies on a related lemma to control
the regret in linear bandits or a similar setting:
\begin{itemize}
\item
  \citet{CarpentierVeAb20} control the quantity \(\sum_{t=0}^{T-1}
  \norm{u_t}_{V_{t+1}^{-p}}\), while we study \(\sum_{t=0}^{T-1} \| u_t
  \|^{2p}_{V_{t}^{-1}} \).
  In addition to the different locations of the \(p\) exponent, they rely on
  $V_0$ being large to show that controlling $\norm{u_t}_{V_{t+1}^{-p}}$ is
  sufficient to bound $\norm{u_t}_{V_t^{-p}}$; this difference in indexing
  precludes more subtle analysis.
\item \citet{ZhangYaJiDu21} generalize the elliptical potential lemma to
  certain structured monotone convex functions to allow a more sophisticated
  analysis, but require a number of boundedness assumptions on actions $A_t$
  and $\hat{\theta}_t$.
\item To adapt to the norm $\norm{\thetaopt}$, \citet{GalesSeJu22} develop
  adaptive algorithms and analyze their regret by counting the number of
  times \(\| u_t \|_{V_{t-1}^{-1}} \) exceeds a threshold, deriving an
  elliptical potential count lemma.
  Because we analyze Thompson sampling, their analyses do not apply here.
\end{itemize}

In contrast, we generalize the standard result
(Lemma~\ref{lem:elliptical-potential}) to remove the requirement that
\(V_0\succeq I_d\) and allow
general exponents $p \in [0, 1]$ in the sum
\(\sum_{t=0}^{T-1} \| u_t\|^{2p}_{V_t^{-1}}\).
To do so, we introduce a \emph{burn-in} term that scales like
\(\tr(V_0^{-p})\) to capture the initial contribution of small eigenvalues
to large values of \(\| u_t \|_{V_t^{-1}}^2\).
The remainder is captured in
the standard \(\log \left(\det V_{T}/ \det V_0\right)\) term.
Before proving the result,
we restate the inequality and give commentary.
In the context of Lemma~\ref{lem:elliptical-potential},
we only require that $V_0 \succ 0$, and obtain
\begin{align}
  \nonumber
  \sum_{t=0}^{T-1} \| u_t \|_{V_t^{-1}}^{2p}
  & \le
  2^p T^{1-p} \left(\log \frac{\det V_{T}}{\det V_0}\right)^p 
  \\
  & \qquad + \frac{3}{2p} \left(
  \tr(V_0^{-p}) - \tr(V_{T}^{-p})
  \right).
  \label{eqn:elliptical-potential}
\end{align}
The second term in inequality~\eqref{eqn:elliptical-potential}
is tight up to within a factor of \(p^{-1}\).
For instance, take the standard basis vectors \(u_t=e_t\) for \(t\in
\{0,\dots,d-1\}\) and \(V_0 = \mathrm{diag}(\lambda_1,\dots,\lambda_d)\);
then
\begin{align}
  \sum_{t=0}^{d-1} \| u_t \|^{2p}_{V_t^{-1}} = \sum_{t=0}^{d-1} \lambda_t^{-p} = \tr(V_0^{-p}).
\end{align}

The first term in inequality~\eqref{eqn:elliptical-potential}
is also tight up to constant factors, even in a regime where the second term is negligible.
To see this, assume for simplicity that \(T\) is an integer multiple of \(d\), take
\( V_0 = (T/d) I_d\),
and define
\[
  u_t
  =
  e^{-1/2} \left(1+\frac{d}{eT}\right)^{\lfloor t/d \rfloor/2} e_{t \bmod d},
  \qquad t=0,\dots,T-1.
\]
Then 
\[ \|u_t\|_2^2 \le e^{-1} \left(1+\frac{d}{eT}\right)^{T/d} \le e^{-1}\exp\!\left(1/e\right) < 1. 
\]

Moreover, each coordinate is updated exactly \(T/d\) times, and after the
\(k\)-th update in a given coordinate, the corresponding diagonal entry of
\(V_t\) equals
\[
  \frac{T}{d} \left(1 + \frac{d}{eT}\right)^k.
\]
Hence, at every time \(t\),
\(
  \|u_t\|_{V_t^{-1}}^2
  =
  d / (eT)
\)
and therefore
\[
  \sum_{t=0}^{T-1} \|u_t\|_{V_t^{-1}}^{2p}
  =
  T \left(\frac{d}{eT}\right)^p
  =
  d^p T^{1-p} e^{-p}.
\]

On the other hand,
\[
  \log \frac{\det V_T}{\det V_0}
  =
  d \cdot \frac{T}{d} \log\!\left(1+\frac{d}{eT}\right)
  =
  T \log\!\left(1+\frac{d}{eT}\right) \le \frac{d}{e}.
\]
Consequently,
\[
  2^p T^{1-p}\left(\log \frac{\det V_T}{\det V_0}\right)^p
  \le 2^p T^{1-p} d^p e^{-p},
\]
which matches the left-hand side up to universal constant factors.

Meanwhile, the second term is upper bounded by
\[
  \frac{3}{2p} \tr(V_0^{-p})
  =
  \frac{3}{2p} d \left(\frac{d}{T}\right)^p
  \left(1 - q^{-pT/d}\right)
  =
  d^{p} T^{1-p} \cdot \frac{3d}{2pT},
\]
which is negligible compared to the first term when $T\gg d$. Hence, the first term
in \eqref{eqn:elliptical-potential} is sharp up to universal constant
factors.
\subsection{Proof of Lemma~\ref{lemma:generalized-elliptical-potential}}

Let $0 < p \le 1$, as the limiting case $p = 0$ follows trivially.
We claim that for all \(t\),
\begin{align}
  \lefteqn{\| u_t \|_{V_t^{-1}}^{2p} \le} \label{eqn:two-bounds-on-norm}
  \\
  & \begin{cases}
    \left(2 \log \frac{\det V_{t+1}}{\det V_t}\right)^p, & \text{if } \| u_t \|_{V_t^{-1}}^2 \le 2, \\
    \frac{3}{2p} (\tr(V_t^{-p}) - \tr(V_{t+1}^{-p})), & \text{if } \| u_t \|_{V_t^{-1}}^2 > 2.
    \end{cases}
  \nonumber
\end{align}
We first verify the claim~\eqref{eqn:two-bounds-on-norm} for the case when
\(\| u_t \|_{V_t^{-1}}^2 \le 2\).
By the matrix determinant formula,
\begin{align*}
  \det V_{t+1} = \left( 1 + u_t^{\top} V_t^{-1} u_t \right) \det V_t ,
\end{align*}
and rearranging gives
\begin{align*}
    \| u_t \|_{V_t^{-1}}^2 = \frac{\det V_{t+1}}{\det V_t} - 1.
\end{align*}
As \(x \le 2 \log (1 + x)\) for \(x \in [0, 2]\), when \( \| u_t
\|_{V_t^{-1}}^2 \le 2\),
\begin{align*}
  \| u_t \|_{V_t^{-1}}^2
  = \frac{\det V_{t+1}}{\det V_t} - 1
  \le 2 \log \frac{\det V_{t+1}}{\det V_t},
\end{align*}
whence the first claim of inequality~\eqref{eqn:two-bounds-on-norm}
follows.

\newcommand{\hermextend}{_{\textup{Her}}}

When $\norm{u_t}^2_{V_t^{-1}} > 2$, we use
an interpolation argument.
By \citet{Lewis96}, as the scalar function \(f(x)=-x^p\) is convex for \(0 < p \le 1\),
the Hermitian extension \(f\hermextend(M) := -\tr(M^p)\)
\begin{enumerate}[label=(\roman*)]
\item is convex on the set of positive definite matrices;
\item when $M$ has spectral decomposition
  $M = U \diag(\lambda(M)) U^*$, then its derivative is
  \begin{align*}
    \nabla f\hermextend(M)
    = U \text{diag} \left( \nabla f(\lambda(M)) \right) U^*.
  \end{align*}
\end{enumerate}
For \(s \in [0, 1]\), define the interpolating function
\begin{align*}
  g(s) = -\tr\left(\left( V_t^{-1} + s \left(V_{t+1}^{-1} - V_t^{-1}\right)  \right)^p\right),
\end{align*}
which satisfies \(g(0) = -\tr(V_t^{-p})\), \(g(1) = -\tr(V_{t+1}^{-p})\),
and
\begin{align*}
  g'(0) & =
  \left\langle \nabla f\hermextend(V_t^{-1}), V_{t+1}^{-1} - V_t^{-1} \right\rangle
  \\
  & = -\tr \left( p V_t^{-(p-1)} \left(V_{t+1}^{-1} - V_t^{-1}\right) \right).
\end{align*}
Since \(f\hermextend(M) = -\tr(M^p)\) is convex in \(M\),
\(g(s)\) is convex in \(s\),
so \(g(1) - g(0) \ge g'(0)\)
by first-order convexity.
Therefore,
\begin{align*}
  \lefteqn{\tr(V_t^{-p})-\tr(V_{t+1}^{-p})} \\
  & = g(1) - g(0)
  \ge g'(0)
  = \tr \left( p V_t^{1-p} \left(V_{t}^{-1} - V_{t+1}^{-1}\right) \right)
  .
\end{align*}
The inversion formula for low rank updates (Sherman-Morrison) implies
\begin{align}
  \nonumber \tr(V_t^{-p})-\tr(V_{t+1}^{-p})
  & \ge
  p\tr\left(
  V_t^{1-p} \frac{V_t^{-1} u_t u_t^{\top} V_t^{-1}}{1 + u_t^{\top} V_t^{-1} u_t}
  \right)
  \\
  & =
  p \cdot \frac{\| u_t \|_{V_t^{-1-p}}^2}{1 + \| u_t \|_{V_t^{-1}}^2}
  \label{eq:trace-difference}
  .
\end{align}
With this, we have almost completed the proof, but we must replace
the ratio in inequality~\eqref{eq:trace-difference} with one
involving $\norm{u_t}_{V_t^{-1}}^{2p}$.

We now apply H\"older's inequality to lower bound
\(\norm{\cdot}_{V_t^{-1-p}}\) by $\norm{\cdot}_{V_t^{-1}}$
via the following claim.
\begin{lemma}
  \label{lemma:holder-algebra}
  Let $V \succ 0$, $\ltwo{u} \le 1$ and
  $\norm{u}_{V^{-1}}^2 \ge 2$, and $0 < p \le 1$.
  Then
  \begin{align*}
    \frac{2}{3} \norm{u}_{V^{-1}}^{2p}
    \le \frac{\norm{u}_{V^{-1-p}}^2}{1 + \norm{u}_{V^{-1}}^2}.
  \end{align*}
\end{lemma}
\noindent
Deferring the proof of Lemma~\ref{lemma:holder-algebra}
temporarily, we complete the proof of the elliptical potential lemma.

Substituting the result of Lemma~\ref{lemma:holder-algebra}
into inequality \eqref{eq:trace-difference} yields
\begin{align*}
  \tr\left(V_t^{-p} - V_{t+1}^{-p}\right)
  & \ge \frac{2p}{3} \| u_t \|_{V_t^{-1}}^{2p}
  .
\end{align*}
Combining the cases for \(\| u_t \|_{V_t^{-1}}^2 \le 2\) and \(\| u_t
\|_{V_t^{-1}}^2 > 2\) in inequality~\eqref{eqn:two-bounds-on-norm}, we
obtain the unconditional bound
\begin{align*}
  \norm{u_t}_{V_t^{-1}}^{2p}
  & \le \left(2 \log \frac{\det V_{t+1}}{\det V_t}\right)^{p}
  \nonumber \\
  & \qquad + \frac{3}{2p} \left(
  \tr(V_t^{-p}) - \tr(V_{t+1}^{-p})
  \right).
\end{align*}
Summing over \(t\in \{0,\dots,T-1\}\) gives
\begin{align*}
  \lefteqn{\sum_{t=0}^{T-1} \| u_t \|_{V_t^{-1}}^{2p}}
  \notag\\
  & \le  \sum_{t=0}^{T-1}
  \left(2 \log \frac{\det V_{t+1}}{\det V_t}\right)^{p}
  + \frac{3}{2p} \left(
  \tr(V_0^{-p}) - \tr(V_{T}^{-p})
  \right)
  \\
  & \le 
  2^p T^{1-p} \left(\log \frac{\det V_{T}}{\det V_0}\right)^p 
  + \frac{3}{2p} \left(
  \tr(V_0^{-p}) - \tr(V_{T}^{-p})
  \right)
\end{align*}
by H\"{o}lder's inequality, completing the proof of
Lemma~\ref{lemma:generalized-elliptical-potential}.

Finally, we return to the proof of Lemma~\ref{lemma:holder-algebra}:
\begin{proof}
  Let $V = W \Lambda W^*$ be the eigen-decomposition of $V$,
  and let $w = W^* u$, so that
  $\norm{u}_{V^{q}} = \norm{w}_{\Lambda}^{q}$ for any power $q$.
  Then
  by H\"{o}lder's inequality,
  \begin{align*}
    \norm{u}_{V^{-1}}^2
    & = \norm{w}_{\Lambda^{-1}}^2
    = \sum_{j=1}^d \frac{w_j^2}{\lambda_j} \\
    & \le \bigg(\sum_{j = 1}^d w_j^2\bigg)^\frac{p}{p + 1}
    \bigg(\sum_{j = 1}^d \frac{w_j^2}{\lambda_j^{1 + p}}\bigg)^\frac{1}{1 + p}
    \\
    & = \ltwo{u}^\frac{2p}{p + 1}
    \norm{u}_{V^{-1-p}}^\frac{2}{p + 1},
  \end{align*}
  where we use $\ltwo{W^*u} = \ltwo{u}$.
  Because $\norm{u}_{V^{-1}}^2 \ge 2$ by assumption,
  we have $\norm{u}_{V^{-1}}^2
  \ge \frac{2}{3} (1 + \norm{u}_{V^{-1}}^2)$, so
  \begin{align*}
    \norm{u}_{V^{-1}}^{2p}
    & = \frac{\norm{u}_{V^{-1}}^{2 p + 2}}{\norm{u}_{V^{-1}}^2}
    \\
    &\le \frac{3}{2}
    \frac{\norm{u}_{V^{-1}}^{2p + 2}}{1 + \norm{u}_{V^{-1}}^2}
    \le \frac{3}{2} \frac{\ltwo{u}^{2p} \norm{u}_{V^{-1-p}}^2}{
      1 + \norm{u}_{V^{-1}}^2}.
  \end{align*}
  The result follows from using that $\ltwo{u} \le 1$.
\end{proof}

\section{Tight lower bounds}
\label{sec:lower-bounds}

In this section, we establish a prior-dependent lower bound to show that,
generally, any policy must suffer a burn-in term in its regret.
We do not quite obtain an instance-specific burn-in that perfectly
matches the upper bounds Theorem~\ref{theorem:main-regret} establishes,
but for ``non-pathological'' priors $\thetaopt \sim \normal(0, \Sigma_0)$,
we will see it is sharp to within logarithmic factors.
To that end, let $\policy$ denote an arbitrary policy, meaning
a mapping from histories $\history_t$ to distributions $\policy(\history_t)$
over actions $A_t$.
The Bayesian regret of policy \(\policy\) is then
\begin{align*}
  \regret^\policy(T) := \sum_{t=0}^{T-1} \Eb \left[
  \thetaopt^{\top} A\subopt - R_{t+1}
  \right],
\end{align*}
where the expectation integrates over the randomness in the policy \(\policy\),
the prior, and the noise.
We establish the following lower bound, which adapts the lower bound for
Gaussian bandits that \citet{RusmevichientongTs10} establish, and whose
proof we defer to Appendix~\ref{sec:lower-bounds-proof}.
\begin{theorem}
  \label{theorem:TS-regret-lower-bound}
  Assume the $d$-dimensional linear-Gaussian
  bandit~\eqref{eqn:linear-gaussian-rewards} with action sets satisfying \(r
  \sphere^{d-1} \subset \mathcal{A}_t \subset r \ball_2^d\) and prior
  \(\thetaopt \sim \normal(0, \Sigma_0)\), where $\Sigma_0$ has eigenvalues
  $\tau_1^2 \ge \cdots \ge \tau_d^2$.
  Then for $T \in \mathbb{N}$ and any policy $\policy$,
  \begin{align*}
    \regret^\policy(T)
    \ge \frac{r}{\pi \ltwo{\tau}}
    \sum_{i = 2}^{d} ((i - 1) \wedge T) \tau_i^2.
  \end{align*}
  Additionally, so long as $\Sigma_0 \succ 0$, for $T \in \mathbb{N}$,
  \begin{align*}
    \regret^\policy(T)
    \ge \sqrt{\frac{2}{\pi}} (d - 1) \sigma \sqrt{T}
    - \frac{\sigma^2}{r} \sqrt{\tr(\Sigma_0)}
    \tr(\Sigma_0^{-1}).
  \end{align*}
\end{theorem}

The theorem highlights two regimes: a ``burn-in'' regime, which depends
strongly on the prior $\thetaopt \sim \normal(0, \Sigma_0)$, and the
long-run regret (though the result holds for all finite samples $T$), which
scales as $\sigma d \sqrt{T}$ essentially independently of the prior.
Specializations can make the theorem
clearer.
When the prior covariance $\Sigma_0$ is
a scaled multiple of the identity,
we obtain the following corollary:
\begin{corollary}
  \label{corollary:identity-covariance}
  Let the conditions of Theorem~\ref{theorem:TS-regret-lower-bound} hold and
  $\Sigma_0 = S^2 I_d$.
  Then for a numerical constant $c > 0$,
  \begin{align*}
    \regret^\policy(T)
    \ge c \, S r d^{1/2} \min\{T, d\}.
  \end{align*}
\end{corollary}
When the prior covariance $\Sigma_0$ has eigenvalues with the ``polynomial''
scaling that $\tau_{d - i}^2 = i^{2 \alpha}$ for some $\alpha \ge 0$, then
$\tr(\Sigma_0) = \ltwo{\tau}^2 \asymp d^{1 + 2 \alpha}$, while $\sum_{i =
  1}^d (d - i) i^{2 \alpha} \asymp d^{2 + 2\alpha}$, which yields a lower
bound matching the upper bounds that Theorem~\ref{theorem:main-regret}
(Corollary~\ref{corollary:main-regret}) provides:
\begin{corollary}
  Let the conditions of Theorem~\ref{theorem:TS-regret-lower-bound} hold,
  $\Sigma_0$ have polynomially scaling eigenvalues, and $T \ge d$.
  Then for a numerical constant $c > 0$,
  \begin{align*}
    \regret^\policy(T)
    \ge c \cdot \frac{r}{\sqrt{\tr(\Sigma_0)}}
    d \tr(\Sigma_0)
    \gtrsim r d \sqrt{\tr(\Sigma_0)}.
  \end{align*}
\end{corollary}
\noindent
In each case, Corollary~\ref{corollary:main-regret}
shows Thompson sampling satisfies
$\regret(T) = \Otilde(d\sigma \sqrt{T}  + rd \sqrt{\tr(\Sigma_0)})$,
while Theorem~\ref{theorem:TS-regret-lower-bound}
shows that for \emph{any} Gaussian prior,  $\regret^\policy(T) \gtrsim
d \sigma \sqrt{T}$
once $T \ge 2 \frac{d \sigma^2}{r^2}
\tr(\frac{1}{d} \Sigma_0)
\tr(\frac{1}{d} \Sigma_0^{-1})^2$.
And whenever the number of steps $T \le r^2 \tr(\Sigma_0)/\sigma^2$,
the second burn-in term dominates the lower bound, showing
that Thompson sampling is near-optimal in most parameter regimes.

\citet[Prop.~2, p.~22]{ZhangHeRi25} provide
a minimax lower bound for linear bandits that also exhibits the
necessity of a burn-in-type term for
deterministic feedback
($\sigma^2 = 0$ in the model~\eqref{eqn:linear-gaussian-rewards}).
In particular, fixing $\Sigma \succ 0$, $S \ge 0$,
and action set $\Ac = \{a \in \R^d \mid \norm{a}_{\Sigma^{-1}} \le 1\}$,
they show for the prior
$\thetaopt \sim \normal(0, \frac{S^2}{d} \Sigma)$
that
\begin{align*}
  \regret^\pi(T)
  \ge \frac{S}{\sqrt{d}}
  \sum_{t = 1}^T \left(\Eb[\ltwo{Z}]
  - \sqrt{t - 1}\right)_+,
\end{align*}
where $Z \sim \normal(0, I_d)$.
Because $\Eb[\ltwo{Z}] \ge \sqrt{2 d / \pi}$ (see
Appendix~\ref{sec:lower-bounds-proof}), their result implies the lower bound
\begin{align*}
  \regret^\pi(d) \gtrsim S d^{3/2}
\end{align*}
in this case, which matches
Corollary~\ref{corollary:identity-covariance} with $r = 1$.
Their lower bound relies on the behavior of the zero-noise Gaussian setting
and the duality relationship between the action set $\Ac$ and prior variance
$\Sigma$, making it so that the results are not always comparable.

\section{Generalization to strongly-log-concave distributions}
\label{sec:generalization-strongly-log-concave}

While the regret bounds Theorem~\ref{theorem:main-regret} provides rely on
Gaussianity, it is relatively straightforward to extend them to a slightly
broader class of distributions whose densities enjoy particular
log-concavity properties.
To do this, we begin with two definitions:

\begin{definition}
  Let \(\Lambda \succeq 0\).
  A function \(f: \Rb^d \rightarrow \Rb\) is \emph{ \(\Lambda\)-strongly
  convex} if \(f(x) - \frac{1}{2} x^{\top} \Lambda x\) is convex in $x$.
\end{definition}

\begin{definition}
  A probability distribution $P$ with density \(p\) on $\Rb^d$
  is \emph{
  \(\Lambda\)-strongly log-concave} if \(x \mapsto -\log p(x)\) is
  \(\Lambda\)-strongly convex.
\end{definition}

With these definitions, consider a linear
bandit~\eqref{eqn:linear-gaussian-rewards}
except that instead of noise $W_t \sim \normal(0, \sigma^2)$,
we assume
the rewards
\begin{align*}
  R_{t+1} = \thetaopt^\top A_t + W_{t+1},
  ~~~
  W_{t+1} \simiid P_W
\end{align*}
where $P_W$ is strongly $\sigma^{-2}$-strongly-log-concave on $\Rb$.
Let $\thetaopt$ have $\Sigma_0^{-1}$-strongly-log-concave prior and assume
as before that the action sets satisfy $\Ac \subset r \ball_2^d$.
As in Theorem~\ref{theorem:main-regret}, we define constants
\begin{align*}
  & C_1(d, T)
  = \sqrt{1 + \frac{\log T}{d}}
  \\
  & C_2(d, T, \sigma, r, \Sigma_0) 
  \defeq
  C_1(d, T)
  \sqrt{\log \left( 1 + \frac{T r^2 \opnorm{\Sigma_0}}{d \sigma^2}
    \right)}.
\end{align*}

Then the following theorem generalizes Theorem~\ref{theorem:main-regret}.

\begin{theorem}
  \label{theorem:thompson-log-concave}
  Let the conditions above hold.
  Then for a numerical constant $C < \infty$,
  Thompson sampling satisfies
  \begin{align*}
    \regret(T)
    & \le
    C \Big[d \sigma \sqrt{T} \cdot C_2(d, T, \sigma, r, \Sigma_0)
    \\
    & \qquad
    + r \sqrt{d}\tr\left( \Sigma_{0}^{\frac{1}{2}} \right) C_1(d, T)
    + \sqrt{\tr\left(\Sigma_0\right)} r\Big].
  \end{align*}
\end{theorem}
The result follows, \emph{mutatis mutandis}, via the same arguments
as those we use to prove Theorem~\ref{theorem:main-regret},
so we defer it to Appendix~\ref{sec:proof:TS-log-concave-regret-bound}.

\section{Conclusions}

Analyzing the regret of Thompson sampling in linear-Gaussian bandits, we
show that a \emph{burn-in} we may attribute to prior diffusiveness both
necessarily appears in regret bounds and decouples additively from the
long-run minimax rate \(\sigma d \sqrt{T}\).
This improves upon existing regret bounds, which scale multiplicatively with
prior diffusiveness, and extends to situations in which the noise is
log-concave.
Limitations, and hence natural areas for extending the approaches
here, include (i) the assumptions of log-concavity of the noise and prior
distributions, (ii) the focus on Bayesian regret, and (iii) the (essentially
consequent) assumption that the prior is well-specified.
Addressing any of these could provide interesting avenues for future work.

\section*{Impact Statement}

This is foundational mathematical work. Its primary impact is educational.

\bibliography{fullbib}
\bibliographystyle{icml2026}

\newpage
\appendix
\onecolumn

\section{Proof of Theorem~\ref{theorem:TS-regret-lower-bound}}
\label{sec:lower-bounds-proof}

To simplify notation, as in the proof of
Theorem~\ref{theorem:main-regret}
define \(\Eb_t[\ \cdot\ ] := \Eb[\ \cdot\ \mid \history_t]\) and
\(\Vb_t[\ \cdot\ ] := \var( \cdot\ \mid \history_t)\)
to be the conditional expectation and variance.
We first adapt Lemma 2.2 of \citet{RusmevichientongTs10} to our setting,
which provides the key instantaneous regret bound that we may massage
into the two lower bounds in the theorem.
In the following lemma, for each \(t\in \mathbb{N}\),
let \((\bm{u}_{t,1}, \dots,
\bm{u}_{t,d})\) be an orthonormal basis of \(\Rb^d\), determined
by $H_t$, such that \(\bm{u}_{t,1}\) is parallel to
\(\mu_t=\Eb_t[\thetaopt]\), that is,
$\bm{u}_{t,1} = \mu_t / \ltwo{\mu_t}$.

\begin{lemma}[Instantaneous Regret Lower Bound]
  \label{lem:instantaneous-regret-lower-bound}
  Let the assumptions of Theorem~\ref{theorem:TS-regret-lower-bound} hold.
  Then for every policy \(\policy\) and \(s \le t \in \mathbb{N}\),
  \begin{align*}
    \Eb_{t} \left[\thetaopt^{\top}A\subopt - \thetaopt^\top A_s \right]
    &\ge \half
      \cdot \Eb_{t}
      \left[
      \frac{r}{\ltwo{\thetaopt}}
      \sum_{i=2}^d
      \left \langle \thetaopt - \Eb_{t}[\thetaopt], \bm{u}_{t, i} \right \rangle^2
      + 
      \frac{\ltwo{\thetaopt}}{r}
      \sum_{i = 2}^d \<A_s, \bm{u}_{t,i}\>^2
      \right]
      .
  \end{align*}
\end{lemma}
\begin{proof}
  As \(r \sphere^{d-1} \subset \mathcal{A}_t \subset r \ball_2^d\) for
  all $t$, the optimal action satisfies \(\thetaopt^{\top} A\subopt = r
  \ltwo{\thetaopt}\).
  Thus for $s \in \mathbb{N}$,
  \begin{align*}
    \thetaopt^{\top}A\subopt - \thetaopt^{\top} A_s
    &=
    r \ltwo{\thetaopt} \left(
    1 - \frac{\thetaopt^{\top}}{\ltwo{\thetaopt}} \frac{A_s}{\ltwo{A_s}}
    \right)
    = \frac{1}{2} r \ltwo{\thetaopt} \ltwo{
      \frac{\thetaopt}{\ltwo{\thetaopt}} -  \frac{A_s}{\ltwo{A_s}}
    }^2
    .
  \end{align*}
  Decomposing the right hand side in terms of the orthonormal basis \(\{
  \bm{u}_{t, i} \}_{i=1}^d\), we obtain
  \begin{align*}
    \thetaopt^{\top}A\subopt - \thetaopt^{\top} A_s
    & = \frac{1}{2} r\ltwo{\thetaopt} \sum_{i=1}^d
    \left \langle \frac{\thetaopt}{\ltwo{\thetaopt}} - \frac{A_s}{\ltwo{A_s}}, \bm{u}_{t, i} \right \rangle^2
    \\
    & \ge \frac{1}{2} r\ltwo{\thetaopt} \sum_{i=2}^d
      \left \langle \frac{\thetaopt}{\ltwo{\thetaopt}} - \frac{A_s}{\ltwo{A_s}}, \bm{u}_{t, i} \right \rangle^2
    \\
    & = \frac{1}{2} \sum_{i=2}^d \left(
      \frac{r}{\ltwo{\thetaopt}} \left \langle \thetaopt, \bm{u}_{t, i} \right \rangle^2
      - 2\left \langle \thetaopt, \bm{u}_{t, i} \right \rangle \left \langle A_s, \bm{u}_{t, i} \right \rangle
      + \frac{\ltwo{\thetaopt}}{r} \left \langle A_s, \bm{u}_{t, i} \right \rangle^2
      \right)
      .
  \end{align*}
  Because the sampling distribution of $A_s \sim \policy(H_s)$ enforces that
  $\policy(H_s)$ is $H_s$-measurable, we observe that
  \(\thetaopt\) and \(A_s\) are independent conditional on
  \(\history_{t}\) for any $t \ge s$,
  and $\{\bm{u}_{t,i}\}_{i=1}^d$ is $\history_{t}$-measurable.
  Thus for each $i \in \{2, \ldots, d\}$,
  $\mu_{t} = \Eb_{t}[\thetaopt]$ is orthogonal to
  $\bm{u}_{t,i}$, and so
  \begin{align*}
    \Eb_{t}
    \left[ \left \langle \thetaopt, \bm{u}_{t, i} \right \rangle
      \left \langle A_s, \bm{u}_{t, i} \right \rangle \right]
    = \left \langle \mu_{t}, \bm{u}_{t, i} \right \rangle 
    \Eb_{t} \left[ \left \langle A_s, \bm{u}_{t, i} \right \rangle \right]
    = 0.
  \end{align*}
  Therefore
  \begin{align*}
    \Eb_{t} \left[\thetaopt^{\top}A\subopt - \thetaopt^{\top} A_s \right]
    &\ge \frac{1}{2}\Eb_{t}
    \left[ \sum_{i=2}^d       \frac{r}{\ltwo{\thetaopt}} \left \langle \thetaopt, \bm{u}_{t, i} \right \rangle^2
      + \sum_{i = 2}^d \frac{\ltwo{\thetaopt}}{r}
      \<A_s, \bm{u}_{t,i}\>^2
      \right]
      .
  \end{align*}
  The desired result follows once we observe that \(\mu_{t}=\Eb_{t}[\thetaopt]\)
  is orthogonal to \(\bm{u}_{t, 2},\dots,\bm{u}_{t, d}\).
\end{proof}

\subsection{Proof of the first lower bound in
  Theorem~\ref{theorem:TS-regret-lower-bound}}

We now prove the first claimed lower bound in the theorem.
First, recall Weyl's inequality on the eigenvalues of low rank perturbation
of matrices:

\begin{lemma}\label{lem:eigen-value-perturbation}
  Denote the (real) eigevanlues of a Hermitian matrix \(A\) by
  \(\lambda_1(A) \ge \cdots \ge \lambda_d(A)\).
  Let $A$ be Hermitian and $E$ a Hermitian rank
  $r$ matrix.
  Then for \(i\in \{1,\dots,d-r\}\),
  \begin{align*}
    \lambda_{i+r}(A + E)
    \le \lambda_i(A) \quad \text{and}\quad
    \lambda_{i}(A + E) \ge \lambda_{i+r}(A).
  \end{align*}
\end{lemma}

Intuitively, as each observation only gives a rank-1 update to the posterior
covariance matrix, we expect to suffer instantaneous regret scaling as $r$
and a term involving the resolved prior covariance $\Sigma_0$ for at
least the first $d$ rounds of the procedure.
The next lemma helps to formalize this intuition.

\begin{lemma}
  \label{lemma:instantaneous-regret-lb}
  Let the conditions of Theorem~\ref{theorem:TS-regret-lower-bound}
  hold
  and the prior variance
  $\Vb_0[\thetaopt] = \Sigma_0$ have eigenvalues
  $\lambda_1 \ge \cdots \lambda_d \ge 0$.
  Then for each $t \in \mathbb{N}$,
  \begin{align*}
    \Eb[\thetaopt^\top A\subopt - R_{t+1}]
    & \ge \frac{r}{\pi \Eb[\ltwo{\thetaopt}]}
    \sum_{i = t + 2}^d \lambda_i,
  \end{align*}
  with the convention that empty sums are zero.
\end{lemma}
\begin{proof}
  Without loss of generality we assume $\lambda_d > 0$, as otherwise
  we simply work in a lower dimensional subspace.
  Recall the orthogonal decomposition
  $\{\bm{u}_{t,i}\}_{i = 1}^d$ of $\R^d$ conditional on $H_t$, where
  $\bm{u}_{t,1} = \mu_t / \ltwo{\mu_t}$ normalizes the expectation
  $\mu_t = \Eb_t[\thetaopt]$.
  We begin by applying Lemma~\ref{lem:instantaneous-regret-lower-bound}
with $s = t$. Since the term
$\sum_{i \ge 2} \langle A_t, \bm{u}_{t,i} \rangle^2$ is nonnegative,
we may discard it to obtain
\begin{align*}
  \Eb[\thetaopt^\top A\subopt - R_{t+1}]
  = \Eb[\thetaopt^\top A\subopt - \thetaopt^\top A_t] \notag 
  \ge
  \frac{r}{2}\,
  \Eb\!\left[
    \frac{1}{\|\thetaopt\|_2}
    \sum_{i=2}^d
    \left\langle
      \thetaopt - \Eb_t[\thetaopt], \bm{u}_{t,i}
    \right\rangle^2
  \right].
\end{align*}
Multiplying both sides by $\Eb[\|\thetaopt\|_2]$ and applying the Cauchy-Schwarz inequality (in the form $\Eb[\sqrt{X}]^2 \le
  \Eb[X/Y] \Eb[Y]$ for all nonnegative random variables $X, Y$) yields
\begin{align*}
  \Eb[\|\thetaopt\|_2]\,
  \Eb[\thetaopt^\top A\subopt - R_{t+1}]
  \ge
   \frac{r}{2} \cdot \Eb
    \left[
      \sqrt{
        \sum_{i=2}^d
        \left \langle \thetaopt - \Eb_t[\thetaopt], \bm{u}_{t, i} \right \rangle^2
      }      
      \right]^2.
\end{align*}
  Define the orthogonal random error $M_t := (I - \bm{u}_{t, 1}
  \bm{u}_{t, {}1}^{\top}) (\thetaopt - \Eb_t[\thetaopt])$
  to satisfy $\ltwo{M_t}^2
  = \sum_{i=2}^d \langle \thetaopt - \Eb_t[\thetaopt], \bm{u}_{t,i}\rangle^2$.
  Then evidently the expected instantaneous regret at time \(t\) has
  lower bound
  \begin{align}
    \label{eq:t-lower-bound}
    \Eb [\thetaopt^{\top}A\subopt - \thetaopt^{\top} A_t ]
    &\ge
    \frac{r}{2\Eb[\ltwo{\thetaopt}]}
      \cdot \Eb[\ltwo{M_t}]^2,
  \end{align}
  while conditional on $\history_t$, the error has normal distribution
  \begin{align*}
    M_t\mid \history_t
    \sim \normal\left(0, (I - \bm{u}_{t, 1} \bm{u}_{t, {}1}^{\top}) \Vb_t[\thetaopt] (I - \bm{u}_{t, 1} \bm{u}_{t, {}1}^{\top})\right).
  \end{align*}

  We now provide a lower bound on $\Eb[\ltwo{M_t}]$ in
  inequality~\eqref{eq:t-lower-bound} using Weyl's inequality.
  For any positive definite $\Sigma$ and
  Gaussian vector $Z \sim \normal(0, \Sigma)$, letting
  $|Z|$ denote the elementwise magnitude of $Z$, Jensen's inequality
  implies
  \begin{equation*}
    \sqrt{\tr(\Sigma)} \ge \Eb[\ltwo{Z}] = \Eb[\ltwo{|Z|}]
    \ge \ltwo{\Eb[|Z|]}
    = \sqrt{\frac{2}{\pi}}
    \sqrt{\tr(\Sigma)}.
  \end{equation*}
  Therefore, recalling the matrix $\bA_t = [A_0 ~ \cdots  ~ A_{t-1}]^\top
  \in \R^{t \times d}$ and the cyclic property of the trace, we have
  \begin{align*}
    \frac{\pi}{2} \cdot \Eb_t[\ltwo{M_t}]^2
    \ge \inf_{\ltwo{u} \le 1}
    \tr\left((I - uu^\top) \Vb_t[\thetaopt]
    (I - uu^\top)\right)
    & = \inf_{\ltwo{u} \le 1}
    \tr\left((I - uu^\top)
    \left(\Sigma_0^{-1} + \sigma^{-2} \bA_t^\top \bA_t\right)^{-1}\right) \\
    & \ge \inf_{\ltwo{u} \le 1}
    \inf_{\textup{rank}(E) \le t,
    E \succeq 0}
    \tr\left((I - uu^\top)
    (\Sigma_0^{-1} + E)^{-1}\right),
  \end{align*}
  where we use the prior variance $\Sigma_0 = \Vb_0[\thetaopt]$.
  Considering the eigenvalue
  decomposition $(\Sigma_0^{-1} + E)^{-1} = V \Gamma V^\top$, where
  the eigenvalues
  $\gamma_1 \ge \cdots \ge \gamma_d > 0$,
  we obtain
  \begin{align*}
    \inf_{\ltwo{u} \le 1}
    \tr\left((I - uu^\top) (\Sigma_0^{-1} + E)^{-1}\right)
    = \sum_{i = 2}^d \gamma_i.
  \end{align*}
  To relate \(\gamma_i\) to the eigenvalues of \(\Sigma_0\), let
\(\mu_1 \le \cdots \le \mu_d\) denote the eigenvalues of
\(\Sigma_0^{-1}+E\). Since \(\Sigma_0^{-1}\) has eigenvalues
\(\lambda_1^{-1} \le \cdots \le \lambda_d^{-1}\) and \(E \succeq 0\)
has rank at most \(t\), Weyl's perturbation formula
  (Lemma~\ref{lem:eigen-value-perturbation}) implies that
\[
  \mu_i \le \lambda_{i+t}^{-1}, \qquad i=1,\dots,d-t.
\]
Since \(\mu_i = \gamma_i^{-1}\), this yields
\[
  \gamma_i \ge \lambda_{i+t}, \qquad i=1,\dots,d-t.
\]
  Thus
  $\frac{\pi}{2} \Eb_t[\ltwo{M_t}]^2
  \ge \sum_{i = 2}^{d - t} \lambda_{i + t}$.
  Putting the pieces together,
  we obtain for $t \le d$ that
  \begin{align*}
    \Eb[\ltwo{M_t}]
    \ge \sqrt{\frac{2}{\pi}
      \inf_{U^\top U = I_{t+1}} \tr\left((I - UU^\top) \Sigma_0\right)}
    = \sqrt{\frac{2}{\pi}
      \sum_{i = t + 2}^d \lambda_i}.
  \end{align*}
  Substitute in inequality~\eqref{eq:t-lower-bound}.
\end{proof}

We may now finalize the proof of
the first lower bound in Theorem~\ref{theorem:TS-regret-lower-bound}.
By Jensen's inequality,
$\Eb[\ltwo{\thetaopt}]
\le \sqrt{\tr(\Sigma_0)}$,
so Lemma~\ref{lemma:instantaneous-regret-lb} implies
\begin{align*}
  \Eb[\thetaopt^\top A\subopt
    - R_{t+1}]
  \ge \frac{r}{\pi \sqrt{\tr(\Sigma_0)}}
  \sum_{i = t + 2}^d \tau^2_i.
\end{align*}
Then summing from $t = 0, \ldots, T - 1$,
we have
\begin{align*}
  \sum_{t = 0}^{T-1}
  \Eb[\thetaopt^\top A\subopt - R_{t+1}]
  \ge \frac{r}{\pi \ltwo{\tau}}
  \sum_{i = 2}^{d} ((i - 1) \wedge T )\tau_i^2.
\end{align*}

\subsection{Proof of the second lower bound in
  Theorem~\ref{theorem:TS-regret-lower-bound}}

We start with the regret lower bound
Lemma~\ref{lem:instantaneous-regret-lower-bound} provides:
letting $t = T$ in the lemma and summing yields
\begin{align*}
  \sum_{t = 1}^T \Eb_T[\thetaopt^\top A\subopt
    - \thetaopt^\top A_t]
  \ge \half \sum_{t = 1}^T
  \Eb_T\left[\frac{r}{\ltwo{\thetaopt}}
    \sum_{i = 2}^d \<\thetaopt - \Eb_T[\thetaopt], \bm{u}_{T,i}\>^2
    + \frac{\ltwo{\thetaopt}}{r}
    \sum_{i = 2}^d \<A_t, \bm{u}_{T,i}\>^2\right].
\end{align*}
Noting that the first term is independent of the index $t$,
taking an outer expectation and using the tower-property,
we obtain the regret bound
\begin{align}
  \regret(T)
  &\ge \frac{1}{2}
  \cdot \Eb
  \left[
    \sum_{i=2}^d
    \left(
    \frac{rT}{\ltwo{\thetaopt}}
    \left \langle \thetaopt - \Eb_T[\thetaopt], \bm{u}_{T, i} \right \rangle^2
    + \sum_{t=0}^{T-1} \frac{\ltwo{\thetaopt}}{r} \left \langle A_t, \bm{u}_{T, i} \right \rangle^2
    \right)
    \right],
  \label{eqn:fiddle-faddle}
\end{align}
valid for any policy.
Conditional on $\history_T$, $\langle \thetaopt - \Eb_T[\thetaopt],
\bm{u}_{T, i} \rangle$ is a normal random variable with mean zero and
variance $\bm{u}_{T, i}^\top V_T^{-1} \bm{u}_{T, i}$.
Let
\begin{align*}
  G_i = \langle \thetaopt - \Eb_T[\thetaopt], \bm{u}_{T, i} \rangle /  \sqrt{\bm{u}_{T, i}^\top V_T^{-1} \bm{u}_{T, i}}.
\end{align*}
Since $G_i \sim \normal(0, 1)$ conditional on $\history_T$, it is also
standard normal unconditionally.
On the other hand, the formula~\eqref{eq:posterior-precision}
for the posterior precision matrix $V_T$ implies that
\begin{align*}
  \sum_{t=0}^{T-1} \left \langle A_t, \bm{u}_{T, i} \right \rangle^2
  & = \bm{u}_{T, i}^\top
  \bigg(\sum_{t = 0}^{T-1} \bA_t \bA_t^\top \bigg) \bm{u}_{T, i}
  = \sigma^2 \bm{u}_{T, i}^\top (V_T - \Sigma_0^{-1}) \bm{u}_{T, i}
  .
\end{align*}
Substituting $G_i$ and this formula into the regret
lower bound~\eqref{eqn:fiddle-faddle} yields
\begin{align*}
  \regret(T)
  &\ge \frac{1}{2}
  \cdot \Eb
  \left[
    \sum_{i=2}^d
    \left(
    \frac{rT}{\ltwo{\thetaopt}} G_i^2 \bm{u}_{T, i}^\top V_T^{-1} \bm{u}_{T, i}
    + \frac{\sigma^2 \ltwo{\thetaopt}}{r} u_{T, i}^\top V_T \bm{u}_{T, i}
    \right)
    \right]
  - \Eb\left[\frac{\sigma^2 \ltwo{\thetaopt}}{r} \sum_{i=2}^d \bm{u}_{T, i}^\top \Sigma_0^{-1} \bm{u}_{T, i}\right]
  .
\end{align*}

We upper bound the last term and lower bound the first.
As $\bm{u}_{T, 1}, \ldots, \bm{u}_{T, d}$ form an orthonormal basis of
$\R^d$,
\begin{align*}
  \sum_{i=2}^d \bm{u}_{T, i}^\top \Sigma_0^{-1} \bm{u}_{T, i}
  \le \tr(\Sigma_0^{-1})
  .
\end{align*}
Applying the AM-GM inequality,
$ab \ge 2 \sqrt{ab}$ for $a, b \ge 0$, to each summand in
the first term, we have
\begin{align*}
  \frac{rT}{\ltwo{\thetaopt}} G_i^2 \bm{u}_{T, i}^\top V_T^{-1} \bm{u}_{T, i}
  + \frac{\sigma^2 \ltwo{\thetaopt}}{r} u_{T, i}^\top V_T \bm{u}_{T, i}
  \ge 2 \sigma \sqrt{T} |G_i| \sqrt{\bm{u}_{T, i}^\top V_T^{-1} \bm{u}_{T, i} \cdot \bm{u}_{T, i}^\top V_T \bm{u}_{T, i}}
  \ge 2 \sigma \sqrt{T} |G_i|,
\end{align*}
where the final inequality follows from Cauchy-Schwarz,
that is, that $\langle x, y\rangle \le \sqrt{x^\top B x} \sqrt{y^\top B y}$
for any positive definite $B$.
Plugging these back in, we have
\begin{align*}
  \regret(T)
  &\ge \sigma \sqrt{T} \cdot \Eb\left[\sum_{i=2}^d |G_i|\right] - \Eb\left[\frac{\sigma^2 \ltwo{\thetaopt}}{r} \tr(\Sigma_0^{-1})\right]
  \ge \sqrt{\frac{2}{\pi}} (d-1) \sigma \sqrt{T} - \frac{\sigma^2 }{r} \sqrt{\tr(\Sigma_0)} \tr(\Sigma_0^{-1})
\end{align*}
because $\Eb[|G_i|] = \sqrt{2/\pi}$ and
$\Eb[\ltwo{\thetaopt}]
\le \Eb[\ltwo{\thetaopt}^2]^{1/2}
= \sqrt{\tr(\Sigma_0)}$.

\section{Technical proofs for Section~\ref{sec:generalization-strongly-log-concave}}\label{sec:proofs-strongly-log-concave}

\subsection{Proof of Theorem~\ref{theorem:thompson-log-concave}}
\label{sec:proof:TS-log-concave-regret-bound}

Our proof relies on concentration properties of random variabels
with log-concave distributions.
We begin with a standard definition of sub-Gaussian
vectors~\citep{Vershynin19}.
\begin{definition}
  Let \(\Sigma\) be positive semidefinite.
  A random vector \(X \in \Rb^d\) is \emph{\(\Sigma\)-sub-Gaussian} if for
  all \(v \in \Rb^d\)
  \begin{align*}
    \Eb \left[\exp \left( v^{\top} (X - \Eb[X]) \right)\right]
    \le \exp \left( \frac{1}{2} v^{\top} \Sigma v \right).
  \end{align*}
\end{definition}
Immediately, we observe that $X \sim \normal(\mu, \Sigma)$ is
$\Sigma$-sub-Gaussian.
Similarly, any random variable with log-concave density is also
sub-Gaussian:
\begin{lemma}
  \label{lemma:strongly-log-concave-subgaussian}
  Let \( \Lambda \) be positive definite
  and $X$ have $\Lambda$-strongly log-concave density.
  Then \(X\) is \(2 \Lambda^{-1}\)-sub-Gaussian.  
\end{lemma}
\noindent
This lemma follows essentially immediately from
\citet[Theorem~3.16]{Wainwright19}, but we include a proof for completeness
in Section~\ref{sec:proof-strongly-log-concave-subgaussian}.
With this, we can fairly straightforwardly demonstrate that any
$\Sigma$-sub-Gaussian random vector concentrates, and relatedly, that
any random vector with log-concave density similarly concentrates.
We have the following technical lemma
(again, we defer the proof; see Section~\ref{sec:proof-sub-gaussian-norms}).
\begin{lemma}
  \label{lemma:sub-gaussian-norms}
  Let $X$ be a $\Sigma$-sub-Gaussian vector.
  Then there exists a numerical constant $C < \infty$ such that
  for all $t \ge 0$,
  \begin{align*}
    \Pb\left(\norm{X}_{\Sigma^{-1}}
    \ge C (\sqrt{d} + t )\right) \le 2 e^{-t^2}.
  \end{align*}
\end{lemma}

We now turn to the proof of Theorem~\ref{theorem:thompson-log-concave}
proper.
Recall that \(\rho\) denotes the prior distribution of \(\thetaopt\), and let \(\prior_{t} = \prior(\cdot \mid \history_y)\) denote the posterior
density of \(\thetaopt\) conditioned on the history \(\history_t\).
Let $p_W$ denote the density of the noise variables $W$.
Then by Bayes' rule, for some constant \(c\) independent of \(\theta\),
\begin{align*}
  \log \prior_t(\theta)
  = c + \log \prior(\theta) + \sum_{s=0}^{t-1} \log p_{W}(R_{s+1} - A_s^{\top} \theta)
  ,
\end{align*}
where \(c = -\log Z_t\) and \(Z_t\) is the normalization constant ensuring
that \(\prior_t\) integrates to one.
Define the inverse variance
\begin{align*}
    V_t \defeq \Sigma_0^{-1} + \frac{1}{\sigma^2}\sum_{s=0}^{t-1} A_s A_s^{\top}
    .
\end{align*}
Since \(\log p_{W}(\cdot)\) is \(\sigma^{-2}\)-strongly concave by
assumption, \(\log p_{W}(R_{s+1} - A_s^{\top} \cdot)\) is \( A_s A_s^{\top} /
\sigma^2\)-strongly concave,
and so the posterior $\prior_t$ on $\thetaopt$ is
$V_t$-strongly log-concave.
  
As in the proof of Theorem~\ref{theorem:main-regret},
for a constant \(\beta\) to be determined later, define the events
\begin{align*}
  E_t(\beta) &:= \left\{ \|\hat{\theta}_t - \thetaopt\|_{V_t} \leq \beta  \right\},
  \\
  E(\beta) &:= \bigcap_{t=0}^{T-1} E_t(\beta).
\end{align*}
We use the regret decomposition~\eqref{eqn:split-regret-via-events},
\begin{align*}
    \regret(T)
    & =
    \underbrace{\Eb \left[
        \bm{1}\{E(\beta)\} \sum_{t=0}^{T-1} (\hat{\theta}_t - \theta\subopt)^{\top} A_t
        \right]}_{(I)}
    + \underbrace{\Eb \left[
        \bm{1}\left\{E(\beta)^{\complement}\right\} \sum_{t=0}^{T-1} (\hat{\theta}_t - \theta\subopt)^{\top} A_t
        \right]}_{(II)}. %
\end{align*}
Then inequalities~\eqref{eq:I-bound} and~\eqref{eq:II-bound}, which rely on
no probabilistic structure of the iterates, imply
\begin{align*}
  \left(I\right)
  \le
  \beta \sqrt{2T} \sigma \sqrt{d \log\left(1 + \frac{r^2 \opnorm{\Sigma_0} T}{
      d \sigma^2}\right)}
  + 3 \beta r \tr\left( \Sigma_{0}^{\frac{1}{2}} \right)
  ~~ \mbox{and} ~~
  (II)
  \le
  r T \sqrt{2 \Pb(E(\beta)^{\complement})\tr\left(\Sigma_0\right)}.
\end{align*}
  
We now simply choose a suitable \(\beta\) for this log-concave setting.
Since \(\log \prior_t(\cdot)\) is \(V_t\)-strongly concave, by
Lemma~\ref{lemma:strongly-log-concave-subgaussian}, \(\thetaopt
\mid \history_t\) is \(2 V_t^{-1}\)-sub-Gaussian (that is, conditional
on the history $\history_t$).
So \(\hat{\theta}_t - \thetaopt \mid \history_t \) is \( 4 V_t^{-1}\)-sub-Gaussian.
Thus,
taking \(t = \sqrt{3 \log T}\) in
Lemma~\ref{lemma:sub-gaussian-norms},
we see that by taking
\begin{align*}
  \beta = O(1)
  \sqrt{d + \log T},
\end{align*}
then a union bound implies
for all \(t\in \{0,\dots,T-1\}\),
\begin{align*}
  \Pb \left[ E(\beta)^{\complement} \right] \le \frac{2}{T^2}.
\end{align*}
Combining the pieces gives the theorem.

\subsubsection{Proof of Lemma~\ref{lemma:strongly-log-concave-subgaussian}}
\label{sec:proof-strongly-log-concave-subgaussian}

As we note, the proof is a more or less trivial adaptation of the
argument to prove Theorem~3.16 of \citet{Wainwright19}.
First, we recall the Pr\'{e}kopa-Leinder inequality~\citep[e.g.][]{Ball97},
which states that if $u, v, w$ are non-negative integrable functions
satisfying the log-concavity-type inequality
\begin{align}
  w(\lambda x + (1-\lambda) y) \ge u(x)^{\lambda} v(y)^{1-\lambda},
  ~~~\mbox{all}~ x, y
  \label{eq:prekopa-leindler-condition}
\end{align}
for some $\lambda \in [0, 1]$, then
\begin{align}
  \int w(x) dx \ge \left( \int u(x) dx \right)^{\lambda}
  \left( \int v(x) dx \right)^{1-\lambda} \label{eq:prekopa-leindler-conclusion}
  .
\end{align}

Moving to the proof proper, assume without loss of generality that \(\Eb[X]
= 0\), since we can always shift the distribution by \(-\Eb[X]\).
Fix \(v \in \Rb^d\),
and define the infimal convolution of
$x \mapsto v^\top x$ and $\frac{1}{4}\norm{\cdot}_{\Lambda}^2$ by
\begin{align*}
  g(y) &:=
  \inf_{x} \left\{ v^{\top} x + \frac{1}{4} (x-y)^{\top} \Lambda (x-y) \right\}
  = v^{\top} y - v^{\top}\Lambda^{-1} v.
\end{align*}

Define \(\psi(x):= - \log p(x)\), which is
\(\Lambda\)-strongly convex by assumption.
To apply inequality~\eqref{eq:prekopa-leindler-conclusion},
the auxiliary functions
\begin{align*}
  w(z) & \defeq p(z) = \exp(-\psi(z)),
  ~~~
  u(x) \defeq \exp(- v^{\top} x - \psi(x)),
  ~~~
  v(y) \defeq \exp(g(y) - \psi(y)).
\end{align*}
We verify that inequality~\eqref{eq:prekopa-leindler-condition} holds for
$\lambda = \half$ with these choices:
\begin{align*}
  \lefteqn{\frac{1}{2} \log u(x)+ \frac{1}{2} \log v(y)
  - \log w\left(\frac{1}{2} x + \frac{1}{2} y\right)} \\
  & =
  \left(- \frac{1}{2} v^{\top} x -\frac{1}{2} \psi(x)\right) + \left(\frac{1}{2} g(y) - \frac{1}{2} \psi(y)\right)
  + \psi \left(\frac{1}{2} x + \frac{1}{2} y\right)
  \\
  & =
  \frac{1}{2} \left(
  g(y)
  -v^{\top} x
  - \frac{1}{4} (x-y)^{\top} \Lambda (x-y)
  \right)
  + \left(
  \psi\left(\frac{x+y}{2}\right) + \frac{1}{8} (x-y)^{\top} \Lambda (x-y) - \frac{1}{2} \psi(x) - \frac{1}{2} \psi(y)
  \right).
\end{align*}
The first term is non-positive by the definition of \(g(y)\), and the second
term is non-positive by the \(\Lambda\)-strong convexity of \(\psi(x)\).
Thus,
\begin{align*}
  \frac{1}{2} \log u(x)+ \frac{1}{2} \log v(y) \le \log w\left(\frac{1}{2} x
  + \frac{1}{2} y\right).
\end{align*}
We may therefore apply inequality~\eqref{eq:prekopa-leindler-conclusion}
to obtain
\begin{align*}
    1 =  \int_{\Rb^d} w(x) dx
    \ge \left( \int_{\Rb^d} u(x) dx \right)^{\frac{1}{2}} \left( \int_{\Rb^d} v(x) dx \right)^{\frac{1}{2}}
    = \Eb \left[ e^{-v^{\top} X}\right]^{\frac{1}{2}} \Eb \left[ e^{g(X)} \right]^{\frac{1}{2}}
    .
\end{align*}
Rearranging, we obtain
\begin{align*}
  \Eb \left[ e^{g(X)} \right]
  \le
  \Eb \left[ \exp\left(-v^{\top} X\right)\right]^{-1}
  \le \exp \left( - v^{\top} \Eb[X] \right)^{-1}
  = 1
\end{align*}
by Jensen's inequality.
Plugging in the formula for \(g(y)\) gives
\begin{align*}
  \Eb \left[ \exp\left(v^{\top} X\right)\right] \le \exp \left( v^{\top} \Lambda^{-1} v \right).
\end{align*}
As \(v \in \Rb^d\) was arbitrary, \(X\) is \(2
\Lambda^{-1}\)-sub-Gaussian.

\subsubsection{Proof of Lemma~\ref{lemma:sub-gaussian-norms}}
\label{sec:proof-sub-gaussian-norms}
Define $\psi_2(x) = e^{x^2} - 1$.
Then the \emph{Orlicz}-$\psi_2$-norm~\citep[e.g.][Ch.~5.6]{Wainwright19}
of a random variable $Y$ is
\begin{equation*}
  \norm{Y}_{\psi_2} \defeq \inf
  \left\{t > 0 \mid \Eb[e^{Y^2 / t^2}] \le 2 \right\}.
\end{equation*}
Following \citet{Wainwright19},
we say stochastic process
vector $\{Y_u\}_{u \in \Rb^d}$ is a
$b$-\emph{Orlicz}-$\psi_2$-process if
\begin{align*}
  \norm{Y_\theta - Y_{\theta'}}_{\psi_2} \le b \ltwo{\theta - \theta'}.
\end{align*}
For a compact set $K \subset \Rb^d$ and metric
$\rho$ on $K$, let $\diam_\rho(K) = \sup_{u,v \in K}
\rho(u , v)$ be the $\rho$-diameter of $K$, and
let $N(\epsilon, K, \rho)$
denote the $\rho$-covering number of $K$ at radius $\epsilon$.
Define the entropy integral
\begin{equation*}
  J(K, \rho) \defeq \int_0^{\diam_\rho(K)}
  \sqrt{\log(1 + N(\epsilon, K, \rho))}
  d\epsilon
\end{equation*}
(where we recall that $\psi_2^{-1}(z) = \sqrt{\log(1 + z)}$).
Then~\citep[Thm.~5.36]{Wainwright19} there exists a finite
numerical constant $C < \infty$ such that
for any $t \ge 0$ and any compact $K$,
any $b$-Orlicz process satisfies
\begin{align}
  \label{eqn:orlicz-concentration}
  \Pb\left(\sup_{u, v \in K} |Y_{u} - Y_{v}|
  \ge C (J(K, b \ltwo{\cdot}) + t)\right) \le 2
  \exp\left(-\frac{1}{b^2} \frac{t^2}{\diam_2^2(K)}\right),
\end{align}
where $\diam_2$ denotes the $\ell_2$-diameter.

We now demonstrate that $Y_u \defeq u^\top \Sigma^{-1/2}X$ is an
$O(1)$-Orlicz process.
Indeed, for any $u, v$, we have
\begin{align*}
  \Eb\left[\exp\left((v - u)^\top \Sigma^{-1/2} X\right)\right]
  \le \exp\left(\frac{1}{2} \ltwo{v - u}^2\right),
\end{align*}
so that $Y_u \defeq u^\top \Sigma^{1/2} X$ is an
$O(1)$-Orlicz-$\psi_2$-process by any of the equivalent definitions of
sub-Gaussianity~\citep{BuldyginKo00}.
On the ball $\ball_2^d$,
the covering number $N(\epsilon, \ball_2^d, \ltwo{\cdot})
\le (1 + \frac{2}{\epsilon})^d$, yielding
entropy integral
\begin{align*}
  J(\ball_2^d, \ltwo{\cdot})
  \lesssim \int_0^2 \sqrt{d \log\Big(1 + \frac{2}{\epsilon}\Big)} d\epsilon
  \lesssim \sqrt{d}.
\end{align*}
Set $b = 1$ in inequality~\eqref{eqn:orlicz-concentration} and
take $K = \ball_2^d$.

\clearpage

\end{document}